\documentclass[sigconf, nonacm]{aamas}

\usepackage{balance} 
\usepackage{algorithm}
\usepackage[noend]{algpseudocode}
\usepackage{graphicx}

\usepackage{amsthm} 
\usepackage{amsmath}
\usepackage{subcaption}
\usepackage{multirow}
\usepackage{pifont}      
\usepackage[title]{appendix}
\usepackage{makecell}
\usepackage{todonotes}

\doi{LIEI2830}

\makeatletter
\gdef\@copyrightpermission{
  \begin{minipage}{0.2\columnwidth}
   \href{https://creativecommons.org/licenses/by/4.0/}{\includegraphics[width=0.90\textwidth]{by}}
  \end{minipage}\hfill
  \begin{minipage}{0.8\columnwidth}
   \href{https://creativecommons.org/licenses/by/4.0/}{This work is licensed under a Creative Commons Attribution International 4.0 License.}
  \end{minipage}
  \vspace{5pt}
}
\makeatother

\setcopyright{ifaamas}
\acmConference[AAMAS '26]{Proc.\@ of the 25th International Conference
on Autonomous Agents and Multiagent Systems (AAMAS 2026)}{May 25 -- 29, 2026}
{Paphos, Cyprus}{C.~Amato, L.~Dennis, V.~Mascardi, J.~Thangarajah (eds.)}
\copyrightyear{2026}
\acmYear{2026}
\acmDOI{}
\acmPrice{}
\acmISBN{}

\title{From User Preferences to Base Score Extraction Functions \\in Gradual Argumentation}

\author{Aniol Civit}
\affiliation{
  \institution{Institut de Robòtica i Informàtica Industrial, CSIC-UPC}
  \city{Barcelona}
  \country{Spain}}
\email{acivit@iri.upc.edu}

\author{Antonio Rago}
\affiliation{
  \institution{King's College London}
  \city{London}
  \country{United Kingdom}}
\email{antonio.rago@kcl.ac.uk}
  
\author{Antonio Andriella}
\affiliation{
  \institution{Institut de Robòtica i Informàtica Industrial, CSIC-UPC}
  \city{Barcelona}
  \country{Spain}}
\email{aandriella@iri.upc.edu}

\author{Guillem Alenyà}
\affiliation{
  \institution{Institut de Robòtica i Informàtica Industrial, CSIC-UPC}
  \city{Barcelona}
  \country{Spain}}
\email{galenya@iri.upc.edu}

\author{Francesca Toni}
\affiliation{
  \institution{Imperial College London}
  \city{London}
  \country{United Kingdom}}
\email{ft@imperial.ac.uk}

\begin{abstract}
    Gradual argumentation is a sub-field of Computational Argumentation from symbolic AI which is attracting attention for its ability to support transparent and contestable AI systems. It is considered a useful tool in domains such as decision-making, recommendation, debate analysis, amongst others. The outcomes in such domains are usually dependent on the arguments' base scores, which must be selected carefully. Often, this selection process requires user expertise and may not always be straightforward. On the other hand, organising the arguments by preference could simplify the task. In this work, we introduce \emph{Base Score Extraction Functions}, which provide a mapping from users' preferences over arguments to base scores. These functions can be applied to the arguments of a \emph{Bipolar Argumentation Framework} (BAF), supplemented with preferences, to obtain a \emph{Quantitative Bipolar Argumentation Framework} (QBAF), allowing the use of well-established computational tools in gradual argumentation. We outline the desirable properties of Base Score Extraction Functions, discuss some design choices, and provide an algorithm for base score extraction. Our method incorporates an approximation of non-linearities in human preferences to allow for better approximation of the real ones. Finally, we evaluate our approach both theoretically and experimentally in a robotics setting, and offer recommendations for selecting appropriate gradual semantics in practice.
\end{abstract}


\theoremstyle{plain}
\newtheorem{defn}{Definition}
\newtheorem{property}{Property}
\newtheorem{proposition}{Proposition}
\newtheorem{axiom}{Axiom}

\theoremstyle{definition}

\newtheorem{example}{Example}
\newtheorem{bsef}{Base Score Extraction Function}

\newcommand{\cmark}{\ding{51}} 

\begin{document}


\pagestyle{fancy}
\fancyhead{}


\maketitle 


\section{Introduction}

    Abstract Argumentation Frameworks (AFs)~\cite{Dung95} have long been known to be a promising 
    method for adaptable decision-making systems~\cite{Amgoud_AAI09}. These 
    frameworks are composed of arguments and relations between them, representing conflicts. Semantics, whether extension-based~\cite{Dung95} or gradual~\cite{AAgradual19}, may then be used to determine the acceptance of the arguments, e.g., as potential decisions to be taken. However, in real-world scenarios involving human reasoning, having only arguments and their relations is rarely sufficient to align the decision towards human intentions. For example, some studies claim that the acceptability of arguments is related to the preferences of the audience to which they are addressed~\cite{Kaci_CMA18}.

    Aligning decision-making systems with human preferences~\cite{Gabriel_MM20}, and also values~\cite{Sierra_aamas19}, has never been more important, given their increasing prevalence in humans' daily lives. Within the field of Computational Argumentation, this has been accomplished by extending AFs into Preference-based AFs (PAFs)~\cite{Amgoud_uai13}, in which the arguments' acceptability is based not only on 
    relations and the chosen semantics, but also on a preference ordering. 

    \begin{figure}
        \centering
        \includegraphics[width=0.85\linewidth]{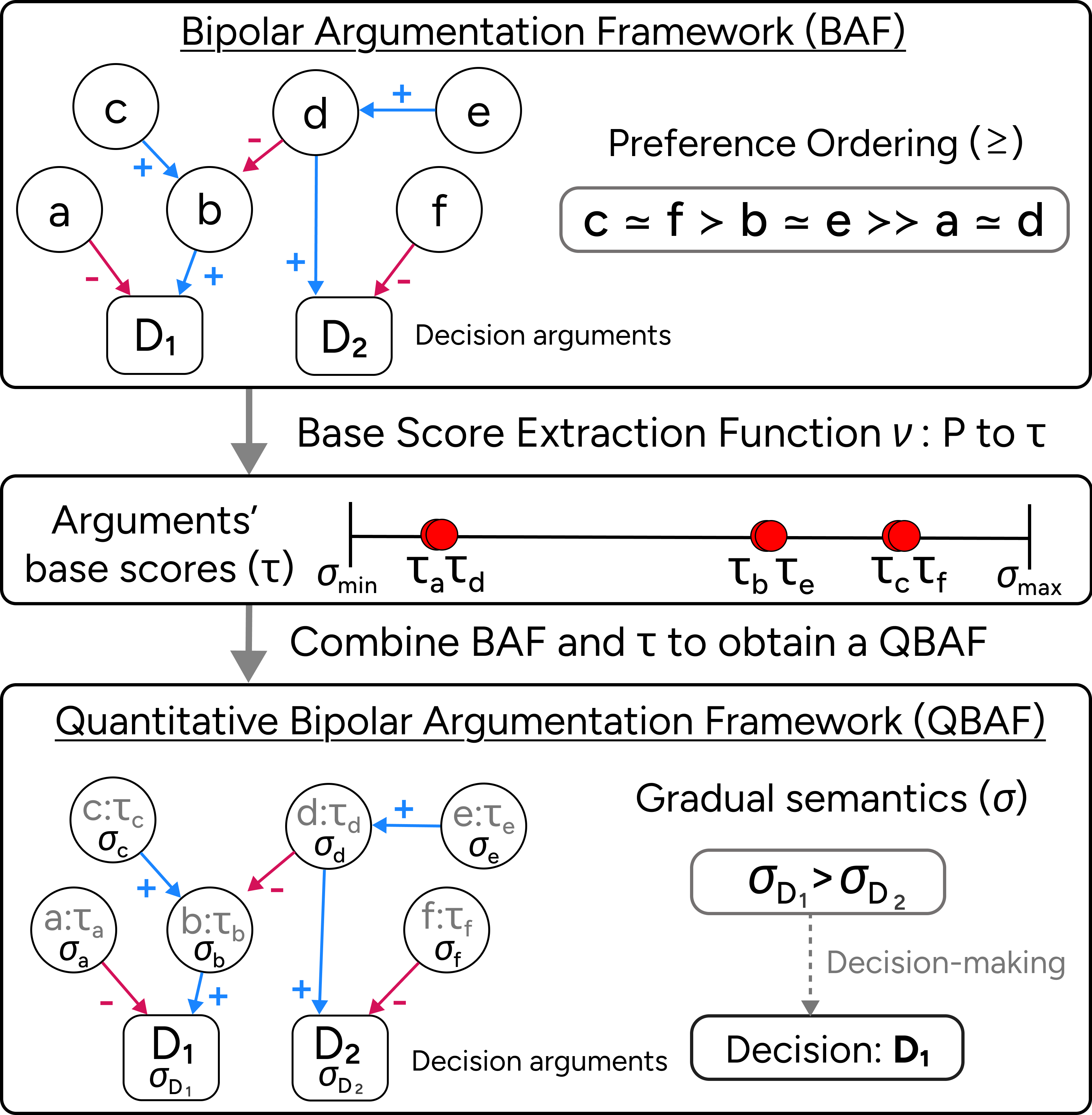}
        \caption{We introduce a methodology for converting a BAF supplemented with a preference ordering of the arguments (top), into a QBAF (bottom). A Base Score Extraction Function $\nu$ is applied to the preference ordering to obtain the arguments' base scores $\tau$ (middle), which allow for the QBAF to be used as normal for decision-making. The output of the decision-making system is adapted to the user's preferences.
        }
        \label{fig:intro_vertical}
        \Description{A BAF appears at the top of the image, with a preference ordering over the arguments. In the middle of the image, a base Score Extraction Function is applied to the preference ordering to obtain the base scores of the arguments. At the bottom of the image, a QBAF is generated using the base scores that were obtained. A decision-making system selects an option between two topic arguments of the QBAF.}
    \end{figure}

    A current limitation in extension-based PAFs is that preferences are used to reduce or modify the framework~\cite{Amgoud_IJAR14, Kaci_CMA18}. Those reductions correspond to different intuitions, and each one is subject to criticism. 
    Moreover, many decision-making settings require a single output, whereas qualitative frameworks may return multiple undominated options or none at all~\cite{Amgoud_AI09}. In contrast, Gradual Argumentation ensures a determinate outcome while remaining transparent about score contributions~\cite{Civit_arxiv25}. In~\cite{Battaglia_icsum24}, reductions are avoided by extending PAFs to \emph{Quantitative Bipolar AFs} (QBAFs) \cite{Baroni_IJAR19}.
    In QBAFs, the arguments can attack and support each other, and each argument is assigned a base score, i.e. an intrinsic strength. Gradual semantics of QBAFs~\cite{Baroni_IJAR19} then assign a final strength to each argument, providing a different form of information than traditional AFs. QBAFs are currently used in several domains, such as decision support~\cite{Battaglia_icsum24}, product recommendation~\cite{Rago_ijcai18}, review aggregation~\cite{Cocarascu_aamas19, Rago_AI25}, and decision-making
    ~\cite{Freedman_aaai25}, given their potential for giving 
    transparent and contestable systems~\cite{Leofante_KR24, Yin_arxiv25}. Nonetheless, in some scenarios, it is often unclear how to set the arguments' base scores. 
    
    In this work, we propose a theoretically and experimentally motivated methodology for generating these base scores based on a given set of user preferences. Our approach is based on adding preferences to \emph{Bipolar Argumentation Frameworks} (BAFs) \cite{Cayrol_ecsqaru05}, i.e., AFs with an additional relation of support (but without a base score). We operate on the hypothesis that the given preference ordering allows for a Base Score Extraction Function (BSEF) to predict the arguments' base scores. 
    Our approach is 
    illustrated 
    in Fig.~\ref{fig:intro_vertical}. 
    Note that, to ensure that it is able to represent realistic user preferences, on top of the usual `I prefer 
    one argument over another', we introduce the possibility of having a preference relation in which one argument is \emph{much more} preferred than another. We then give some general axioms that any BSEF should satisfy. Next, we provide different design choices for adapting the base score extraction, which are formally characterised as potentially desirable, optional properties. Finally, we introduce some concrete BSEFs and evaluate them through both a theoretical analysis and experimentation on a proposed use case, namely \emph{assistive feeding in robotics}. 

    Our contributions are as follows: (i) the introduction of BSEFs, which extract the base scores of a set of arguments given a preference ordering, extended to non-linear preferences; (ii) the definition of a set of axioms and properties for BSEFs; (iii) the introduction of two concrete BSEFs; and (iv) a theoretical and experimental evaluation of the two BSEFs.

    To foster transparency and reproducibility, we have released the full implementation 
    at \url{https://github.com/acivit/From-Preferences-To-Base-Score-Extraction-Functions}.


\section{Preliminaries}

    In this section, we recall the definitions of BAFs, QBAFs and their gradual semantics. 


    A BAF~\cite{Cayrol_ecsqaru05} is a triple $\mathcal{B}=\langle \mathcal{X}, \mathcal{R}^-, \mathcal{R}^+ \rangle$, consisting of a finite set of arguments $\mathcal{X}$, a binary relation of attack $\mathcal{R}^- \subseteq \mathcal{X} \times \mathcal{X}$, and a binary relation of support $\mathcal{R}^+\subseteq \mathcal{X} \times \mathcal{X}$. 

    A QBAF~\cite{Baroni_IJAR19} is a quadruple $ \mathcal{Q}=\langle \mathcal{X}, \mathcal{R}^-, \mathcal{R}^+, \tau\rangle$, where $\langle \mathcal{X}, \mathcal{R}^-, \mathcal{R}^+ \rangle$ is a BAF and $\tau: \mathcal{X} \to [0,1]$ is a base score function. For any $a \in \mathcal{X}$, $\tau(a)$ is the base score of $a$.

    In this work, we use the standard range for both the base scores and argument strengths, i.e., $[0,1]$, though others exist~\cite{Amgoud_kr16}. The strength of an argument $a \in \mathcal{X}$ is given by $\sigma(a)$, where $\sigma: \mathcal{X} \to [0,1]$ 
    is a gradual semantics~\cite{Baroni_IJAR19}.

    As mentioned, we use QBAFs for decision-making. In this context, the possible options of the decision process are included in the QBAF. These are called decision arguments. The QBAF is represented as a set of trees, with the root of each tree corresponding to a decision argument~\cite{Rago_kr23}.
    (We leave cyclic QBAFs to future work, as in \cite{Potyka_aaai21, Rago_kr23, Rago_AI25}.)
    Let $\mathcal{Q}$ be a QBAF $\langle \mathcal{X}, \mathcal{R}^-, \mathcal{R}^+, \tau\rangle$. For any arguments $a,b\in \mathcal{X}$, let a path from $a$ to $b$ be defined as a sequence 
    $(\gamma_0,\gamma_1),...,$ $(\gamma_{n-1},\gamma_n)$ of length $n>0$, where $\gamma_0=a$ and $\gamma_n=b$, and, for any $1\leq i \leq n, (\gamma_{i-1}, \gamma_i)\in \mathcal{R}^+ \cup \mathcal{R}^-$. Then, given a set of decision arguments $D\subseteq\mathcal{X}$, $\mathcal{Q}$ is a QBAF for $D$ iff i) $\nexists a\in \mathcal{X}\setminus \{ D\}$ such that $\exists d\in D$ where $(d, a)\in \mathcal{R}^+\cup\mathcal{R}^-$ ii) $\forall a\in \mathcal{X} \setminus \{D \}$ there is a path from $a$ to at least one $d \in D$; and iii) $\nexists a \!\in \! \mathcal{X}$ with a path from $a$ to $a$. We set the base score of the decision arguments to 0.5.
    
    In the remainder of the paper, unless otherwise specified, we will assume a BAF $\langle \mathcal{X}, \mathcal{R}^-, \mathcal{R}^+ \rangle$.

    As a running example, we will use robot-assisted feeding. Here, a robot feeds a user who cannot eat independently and must decide the pace of feeding, either slow or fast. The robot uses a BAF drawing arguments from both the user and the robot's observations, e.g., user status and history. 
    The arguments are: $(a)$ Eating slowly causes boredom; $(b)$ eating slowly does not stress the patient out; $(c)$ the patient is vulnerable and stress must be avoided; $(d)$ eating slowly reduces time with a visiting niece; $(e)$ the patient wants to tell their niece something important, and doing so will help the patient to feel more relaxed; and $(f)$ eating quickly carries a (low) risk of dysphagia. The decision arguments are \textit{slow} ($D_1$) and \textit{fast} ($D_2$). 
    The BAF is represented in Fig.~\ref{fig:intro_vertical}.
    The aim is to select the best option, namely any decision argument with the highest final strength in the QBAF obtained after predicting the arguments' base scores to adhere to the preferences (again, as in Fig.~\ref{fig:intro_vertical}). 


\section{Extracting Base Scores from Preference Orderings}
    
    In this section, we show that adding preferences to BAFs allows us to define a BSEF, which may exhibit some defined desirable axioms and properties. Then, the BSEF is used to extract the arguments' base scores, effectively converting the BAF into a QBAF.

    \subsection{Adding Preferences to BAFs}

        Adding preferences over arguments is a common approach to personalising argumentation frameworks for users~\cite{Modgil_AI13, Amgoud_uai13, Mailly_kr20}. We 
        formally define preferences over arguments as follows:

        \begin{defn}\label{def:pref_ordering}
            A preference $\succeq
            $ over $\mathcal{X}$ is a reflexive and transitive relation on $\mathcal{X}$. Given $a,b\in \mathcal{X}$, $a\succeq b$ and $a\npreceq b$ will be denoted as $a\succ b$, while $a\succeq b$ and $a\preceq b$ will be denoted as $a \simeq b$. 
        \end{defn}
        

        Setting the base scores according to the preference ordering of a BAF's arguments in a decision-making environment can give an intuition of the most suitable decisions in line with a user's preferences. The following examples show the importance of properly setting the arguments' base scores: 

        \begin{example}
        \label{ex:introduction}
            Consider the framework from Fig.~\ref{fig:intro_vertical}. If the base scores are arbitrarily set to 0.5, giving equal importance to all arguments, using the Quadratic Energy (QE) Model~\cite{Potyka_kr18}, the final strengths of the decision arguments are $\sigma(D_1)=0.5$ for slow, and $\sigma(D_2)=0.505$ for fast. Then the robot would choose to move fast.
        \end{example}

        \begin{example}
\label{ex:introduction_preferences}
            (Example~\ref{ex:introduction} Cont.) 
           Imagine a scenario where a user gives more importance to the arguments related to its safety (e.g., $c$ and $f$), less importance to 
        those related to its relaxation (e.g., $b$ and $e$), and the least importance to the arguments related to its enjoyment (e.g., $a$ and $d$).
        The resulting preference ordering is: $c\simeq f \succ b\simeq e \succ a \simeq d$. Intuitively, the base scores of $c$ and $f$ should be greater than those of $b$ and $e$, and 
        which should in turn be greater than the base scores of $a$ and $d$. For example, they could be set to $\tau(c)=\tau(f)=0.75$, $\tau(b)=\tau(e)=0.5$, and $\tau(a)=\tau(d)=0.25$. Here, using the QE Model again, the final strengths are $\sigma(D_1)=0.54$ and $\sigma(D_2)=0.45$. Then the robot would choose to move slowly, which is the opposite of what was selected earlier.

        \end{example}

    \subsection{Adapting to Non-Linear Human Preferences}

        Human preference relations are generally not linear~\cite{Northcraft_OBHDP98}. People's choices and priorities often fluctuate instead of increasing or decreasing in a simple, linear manner as conditions change. This non-linearity makes it complex to set the base scores properly. A possible approach to including these non-linearities could be achieved by allowing the user to order their preferences gradually within a range. However, setting a gradual score is usually slower and has more sample noise than cardinal orderings~\cite{Yannakakis_acii11, Shah_arxiv14}. 
                    
        In this work, we approximate the non-linearities by introducing a new preference relation, where a user may have a \emph{much greater} ($\succ\!\succ$) preference for one argument over another. This relation considers the non-linearity and is easy to recognise from human feedback. Definition~\ref{def:pref_ordering} is extended as follows:

        \begin{defn}
            Given $a,b\in\mathcal{X}$, if $a$ is much more preferred than $b$, the preference relation will be denoted as $a \succ\!\succ b$. This relation is transitive, and 
            $a \succ\!\succ b$ implies
            $a \succ b$. Therefore, a preference relation between arguments $a,b\in\mathcal{X}$ in a preference ordering $\succeq
            $ can be:
            \begin{itemize}
                \item $a\simeq b$, which denotes indifference;
                \item $a\succ b$, which denotes strict preference;
                \item $a\succ\!\succ b$, which denotes a much stronger preference.   
            \end{itemize}
        \end{defn}

        \begin{example}\label{ex:much_more_preferred}
            Given the arguments $a,b,c\in \mathcal{X}$, with a preference ordering: $a \succ\!\succ b \succ c$, we may expect that $\tau(a)$ is \emph{much greater} than $\tau(b)$ and $\tau(c)$, e.g., 0.9, 0.3, and 0.1, respectively. 
        \end{example}




        This new relation breaks the linearity in the preference ordering, providing a more realistic approximation to human preferences, since it has a greater expressiveness.


    \subsection{Base Score Extraction Functions}

        Although PAFs and BAFs offer means to capture the evaluation of arguments, they remain categorical and context-sensitive. In contrast, integrating a method to extract base scores from preference orderings would provide a more granular and objective evaluation mechanism. Furthermore, the literature on QBAFs offers methods to provide transparent and counterfactual explanations~\cite{Kampik_IJAR24, Yin_kr24} and allows for fine-tuning the framework to better adapt to users, which is crucial in some decision-making systems, such as robotics.  

        
        Therefore, we introduce BSEFs that compute the arguments' base scores given a preference ordering. 

        \begin{defn}\label{def:BSEF}
            A BSEF $\nu$ is a function that 
            maps 
            a set of preference orderings $\mathcal{P}$ to 
            a set of 
            base score functions $\mathcal{T}$ (for arguments in $\mathcal{X}$) 
; namely $\nu: 
            \mathcal{P} \rightarrow \mathcal{T}$.
        \end{defn}
        

        Obtaining the base scores of a set of arguments allows the transformation of BAFs into QBAFs. 

    \subsection{Desirable Properties of Base Score Extraction Functions}

        In the following, we present the different axioms and properties of the BSEFs. First, we define the \textit{Preference Coherence} axiom, which states that if one argument is preferred over another, the former's base score will be greater.

        \begin{axiom}\textnormal{(Preference Coherence)}
            Given a preference ordering $\succeq
            $, a BSEF $\nu
            (
            \succeq)=\tau$ satisfies \textnormal{Preference Coherence} iff, for 
            any arguments $a, b\in \mathcal{X}$:
            \begin{itemize}
                \item if $a\succ b$ or $a\succ\!\succ b$ then $\tau(a) > \tau(b)$; and
                \item if $a\simeq b$ then $\tau(a) = \tau(b)$.
            \end{itemize}
        \label{axiom:1}
        \end{axiom}

        We now define \textit{Preference Relation Coherence}, which provides a differentiation between the preference relations presented. Intuitively, if an argument is much preferred ($\succ\!\succ$) over another, the difference of their base scores will be greater than if the preference relation is only preferred ($\succ$). 
        
        \begin{axiom}\textnormal{(Preference Relation Coherence)}
        \label{property:relation_coherence}
            Given a set of arguments arguments $a, b, c, d\in \mathcal{X}$ and the preference ordering $a\succ\!\succ b\succ c \simeq d$, a BSEF $\nu(\succeq)=\tau$ satisfies \textnormal{Preference Relation Coherence} iff $\tau(a)-\tau(b)>\tau(b)-\tau(c)>\tau(c)-\tau(d)$.
        \label{axiom:2}
        \end{axiom}

        In our approach, the focus on extracting the base scores is set on the preference structure rather than the comparison between elements. The following axiom is necessary to define whether two or more preference orderings are isomorphic. If the preferences of two sets of arguments exhibit the same relationships, they can be said to have isomorphic preferences.  

        \begin{defn}
            Let 
            $ \mathcal{X}'$  be a set of arguments 
            and $\succeq
            $ and $\succeq
            '$ be two preference orderings 
            over $\mathcal{X}$ and $\mathcal{X}'$, respectively.
             Those orderings are isomorphic, denoted $\succeq
            \simeq \succeq
            '$, iff there is a bijective function $f: \mathcal{X} \to \mathcal{X
            }'$ such that $\forall i,j \in \mathcal{X}, \ i\succeq
            j \Leftrightarrow f(i)\succeq
            'f(j)$.
        \end{defn}

        We require that a BSEF preserves the structural distinction of users' preferences. If two users have different (non-isomorphic) preference orderings, the assigned base scores must reflect it.
        
        \begin{axiom}
            \textnormal{(Preference Structure Coherence)} Given two preference orderings $\succeq$ and $\succeq'$ (over the same $\mathcal{X}$) that are not isomorphic ($\succeq
        \not\simeq \succeq
            '$), then a BSEF $\nu$ 
            satisfies \textnormal{Preference Structure Coherence} iff 
            $\nu
            (\succeq) \!\neq \!\nu(\succeq')$.
        \label{axiom:3}
        \end{axiom}

        Note that if two sets of arguments have different base scores, they will not necessarily have different preference ordering.



    \subsection{Design Choices for Base Score Extraction Functions}

        The presented axioms offer the necessary conditions that the BSEFs must satisfy, but determining the base scores is still a flexible and context-dependent challenge. This section outlines different design choices that must be considered when designing a BSEF.

        \subsubsection{Setting the Range}
            A choice that must be made is considering where the base scores of the most and least preferred arguments should be placed. Establishing them as 1 and 0, respectively, determines a strict and high importance of the preferences, but also allows for possible saturations from incoming supports or attacks, which would provoke information loss, an existing concern in PAFs. For a further investigation, we define the base score limits.       
        \begin{defn}\label{def:top_bs}
            The base score of the most preferred arguments from the BSEF is defined as $\nu
            (\succeq)(x) = \top$ for all $x \in  max_{\succeq
            }(\mathcal{X})$.\footnote{$max_{\succeq
            }(\mathcal{X})=\{x\in \mathcal{X} \mid \nexists y\in \mathcal{X} \ s.t. \ y \succ x \ or \ y\succ\!\succ x \}$.}
        \end{defn}
        \begin{defn}\label{def:bot_bs}
            The base score of the least preferred arguments from the BSEF is defined as $\nu
            (\succeq)(x) = \bot$ for all $x \in min_{\succeq
            }(\mathcal{X})$.\footnote{$min_{\succeq
            }(\mathcal{X})=\{x\in \mathcal{X} \mid \nexists y\in \mathcal{X} \ s.t. \ x \succ y \ or \ x\succ\!\succ y \}$.}
        \end{defn}

        The range is constrained to that of the gradual semantics, i.e. $
        [0,1]$. Namely, it must be satisfied that $0 \leq \bot \leq \top \leq 1$. A BSEF will be considered normalised if its limits are set to $\top=1$ and $\bot=0$:

        \begin{property}\label{property:normalisation}
            \textnormal{(Base Score Normalisation)} A BSEF $\nu$ satisfies \textnormal{Base Score Normalisation} iff $\top=1$ and $\bot=0$.
        \label{property:1}
        \end{property}

        In some scenarios, the limit base scores might be centred at 0.5 to differentiate between important and less important arguments. 

        \begin{property}\label{property:centralisation}
            \textnormal{(Base Score Centralisation)} A BSEF $\nu$  
            satisfies \textnormal{Base Score Centralisation} iff $\top=1-\bot$.
        \label{property:2}
        \end{property}
    
        The range influences the base scores' compression. For example, if the range is too low, the base scores would be very similar. Hence, the preferences of the arguments would not be properly reflected. The following example illustrates this concern.

        \begin{example}
            Given the arguments $a,b,c,d,e,f \in \mathcal{X}$, consider the preference ordering: $c\simeq f \succ b\simeq e \succ a \simeq d$. If the limits were set to $\top=0.9$ and $\bot=0.85$, the base scores could be placed only between $[0.85,0.9]$ (e.g.,  $\tau(c)=\tau(f)=0.9, \tau(b)=\tau(e)=0.875, \tau(a)=\tau(d)=0.85$). 
        Although these base scores 
        satisfy all the axioms, they may not sufficiently differentiate the preferences enough for some applications, since the difference between them is too low.
        \end{example}

        \subsubsection{Preference Regularity}
            
            In many real-world cases, users can express their preferences regularly, meaning that the perceived difference in importance between two adjacent preferred arguments is roughly uniform across the ordering. For instance, if a user slightly prefers argument $a$ over $b$, and $b$ over $c$, the gap between $a$ and $b$ is comparable to that between $b$ and $c$. Assuming such regularity simplifies the computation of base scores.             
            Before introducing the property, we define the notion of adjacent arguments. 
    
            \begin{defn}
                Let $a,b\in\mathcal{X}$ 
                 and $\succeq
                $ be a preference ordering
                .
                The arguments $a,b$ are adjacent in 
                $\succeq
                $ 
                iff one of the following conditions holds:
                \begin{itemize}
                    \item $a\succ b$ and $\nexists c$ such that $a \succ c \succ b$;
                    \item $a\succ\!\succ b$ and $\nexists c$ such that $a \succ\!\succ c \succ b$ or $a \succ\!\succ c \succ\!\succ b$ or $a \succ c \succ\!\succ b$;
                    \item $a \simeq b$.
                \end{itemize}
            \end{defn}
    

    \begin{property}\label{property:regularity}
                \textnormal{(Base Score Relation Regularity)} A BSEF $\nu$ satisfies \textnormal{Base Score Relation Regularity} iff 
                for any preference relation 
                $\succeq$ and any two pairs
                $(a,b)$ and $(c,d)$ of adjacent arguments in 
                $\succeq$, 
                 then for $\tau=\nu(\succeq)$ it holds that $\tau(a)-\tau(b)=\tau(c)-\tau(d)$.
                \label{property:3}
            \end{property}
            
            \begin{example}
                \label{ex:BSRR}(Example~\ref{ex:introduction} Cont.) Given 
                the preference ordering: $c\simeq f \succ b\simeq e \succ a \simeq d$, let the limits be $\top=0.9$ and $\bot=0.1$. If the BSEF $\nu$ satisfies Base Score Relation Regularity, then any base score function
                obtained from $\nu$
                would give
                base scores $\tau(c)=\tau(f)=0.9, \tau(b)=\tau(e)=0.5, \tau(a)=\tau(d)=0.1$. Otherwise, the base scores $\tau(b)$ and $\tau(c)$ could be arbitrarily placed between 0.1 and 0.9. 
            \end{example}


            If this regularity condition is not enforced, preference intensity could be encoded via irregular spacing between arguments rather than using a separate \emph{much greater} relation. Yet, this flexibility makes the base score extraction process more difficult, as there is no clear rule for assigning consistent numerical differences.

            Another property is that adding new arguments equally preferred to existing ones does not change the base scores of the other arguments.

            \begin{property}
            \label{property:base_score_preference_stability}
                \textnormal{(Base Score Preference Stability)} Let $\succeq
                $ be a preference ordering
                ,
                $S$ a set of arguments such that $S \cap \mathcal{X}=\emptyset$, $\mathcal{X}'=\mathcal{X}\cup S$, and 
                $\succeq'$ be a preference ordering over $\mathcal{X}'$ such that $\forall s\in S$, $\exists a\in \mathcal{X}$ such that $a \simeq s$ and,
                $\forall a,b\in \mathcal{X}$,
                $a\succeq' b$ iff $a\succeq b$. 
                Then BSEFs $\nu$ (as per Definition~\ref{def:BSEF})
                and $\nu'$
                (mapping a set of preference orderings over $\mathcal{X}'$ to a set of base scores for arguments in $\mathcal{X}'$) satisfy \textnormal{Base Score Preference Stability} iff 
                for $\tau=\nu
                (
                \succeq)$ and  $\tau'=\nu
                '(
                \succeq')$,
                for all $x \in \mathcal{X}$, it holds that $\tau(x)=\tau'(x)$.              \label{property:4}
            \end{property}

        \subsubsection{Preferences ratios}

            Users often express varying intensities in their preferences. For instance, one argument may be only slightly preferred over another, while another is much more important. To model such heterogeneous gaps, we introduce the concept of preference ratios, which quantify how much stronger a \emph{much greater} preference should be compared to a regular preference ($\succ$). This ratio provides a systematic way to translate qualitative intensity (e.g., ``much more important'') into quantitative differences between base scores.

            \begin{example}
                (Example~\ref{ex:BSRR} Cont. ) Now the preference ordering becomes: $c\simeq f \succ\!\succ b\simeq e \succ a \simeq d$. From asking the human, the ratio of the \emph{much greater} preference is set to 3. The BSEF limits are set to 0.9 and 0.1. Then, the base scores could be set to $\tau(c)=\tau(f)=0.9, \tau(b)=\tau(e)=0.3, \tau(a)=\tau(d)=0.1$. Here, the difference between preferred arguments is 0.2, while for much preferred arguments is 0.6, hence the ratio is respected. 
    
            \end{example}

            Observe that the ratio between the different preferred relations and the limits must be placed carefully, as there is a risk of undermining the preferred or much more preferred relations. The following example shows two scenarios where this might occur:
    
            \begin{example}
                Consider $a,b,c,d \in \mathcal{X}$, with the preference ordering: $a\succ b \succ\!\succ c \succ d$. If the ratio is set to 1.33, with the limits at 0.75 and 0.25, the base scores would be $\tau(a)=0.75, \tau(b)=0.6, \tau(c)=0.4, \tau(d)=0.25$. The difference between preferred arguments is 0.15, while it is 0.2 for much more preferred arguments. At first sight, it does not seem representative enough. Alternatively, setting a ratio of 96 and the limits at 0.99 and 0.01, the base scores would be $\tau(a)=0.99, \tau(b)=0.98, \tau(c)=0.02, \tau(d)=0.01$. Here, the difference between the adjacent preferred arguments is 0.01, which intuitively may not be representative in many real-world contexts. 
            \end{example}
    
            These examples highlight that the preference ratio serves as a tuning parameter controlling how sensitively the system reacts to \emph{much greater} preferences. A well-chosen ratio helps capture meaningful differences in user priorities while maintaining numerical stability in base score extraction. The combination of these three different design choices thus gives a formal representation of how BSEFs may be personalised to different settings.\looseness=-1

\section{Concrete Base Score Extraction Functions}

    In this section, two BSEFs are introduced. Those functions are based on the preference ordering of the arguments and allow for the presented design choices. An algorithm is developed to compute the base scores. Some examples of its use are shown using the running example. By definition, as the BSEFs are based on preference orderings, then, scenarios with only one argument are not considered. 
    
    Since preference orderings are monotonic, considering a descending preference ordering (from the most preferred to the least preferred), the BSEF should also be monotonically decreasing.  

    \begin{defn}
        A BSEF $\nu$ is monotonically decreasing (increasing) iff 
        for any 
        base score function $\tau=\nu(\succeq)$, 
        given two arguments $a,b$ where $a\succeq b$, then 
        $\tau(a)
        \leq 
        \tau(b)$ ($
        \tau(a)\geq 
        \tau(b)$, respectively).  
    \end{defn}

    \begin{proposition}\label{prop:1}
        A monotonically increasing BSEF violates Axiom 1 for a descending preference ordering.\footnote{The proof of this proposition is found in~\cite{Civit_arxiv26}.}
    \end{proposition}


    Creating a monotonically decreasing function directly from the preference ordering is possible. The following method consists of assigning a distance $d(x)$ 
    to each argument $x\in\mathcal{X}$ according to the preference ordering and then normalising it between the desired range. The distance $d(x)$ is relative to the most preferred arguments, and the least preferred arguments are at the maximum absolute distance $D$ to the most preferred arguments.
    \begin{bsef}
    \label{bsef:1}
        (Adaptable range) Our first BSEF allows for the range to be set and is defined as follows.
        For any ordering $\succeq$ and $x \in \mathcal{X}$, let 
        $\top$ be the base score for the most preferred arguments (see Definition~\ref{def:top_bs}) and
        $\bot$ be the base score of the least preferred arguments (see Definition~\ref{def:bot_bs}). Then:%
    \end{bsef}
%
    \begin{equation}\label{eq:base_score_function_margins}
        \nu_1(\succeq)(x) = 
        \bot + (
        \top - 
        \bot) \cdot \frac{D-d(x)}{D-1}
    \end{equation}

With the restriction $0 \leq \bot 
    \leq \top 
    \leq 
    1$ and $D>1$, 
    the base score functions obtained with this BSEF are well-defined, giving base scores inside $[0,1]$.
    The edge cases 
    occur at $d(x)=1$ (for the most preferred arguments) 
    and at $d(x)=D$ (for the least preferred arguments)
    , yielding $v_1(\succeq)(x)=\top$ and $v_1(\succeq)(x)=\bot
    $ respectively. 
    
    The next BSEF consists of setting the compression of the base scores and the distances towards the edge cases. 
    \begin{bsef}
    \label{bsef:2}
        (Adaptable squeezing and distancing) Our next BSEF allows a parameter to determine how compressed the base scores are to be set, and another to determine how distant the edges are, and is defined as:
    \end{bsef} %
    \begin{equation}
    \label{eq:base_score_function_ranking_a_b}
        \nu_2(
        \succeq)(x) = \frac{D-d(x)+\alpha}{D-1+\beta}
    \end{equation}
    where $\beta\in[0,+\infty)$ controls the compression of the base scores, and $\alpha\in[0,+\infty)$ determines the distance to the edges. Intuitively, larger values of $\beta$ yield more compressed base scores, and larger values of $\alpha$ place the most preferred argument closer to 1. 
    Note that 
    it is necessary that $\alpha \leq \beta$, otherwise the base scores would be greater than 1. With $\beta=D-1$, the spread of the base scores is compressed in a range of 0.5. 
    This BSEF is better suited to cases where the number of arguments is unknown or likely to change, whereas the first direct control of the range and the edges allows for the range to be set directly, which provides more control over the edge cases. \looseness=-1

    Algorithm~\ref{alg:flexible_base_score_extraction} shows how to extract the base scores from a set of arguments $\mathcal{X}$ and its preference ordering $\succeq$. We assume that there is at least one pair of arguments related by $\succ$ or $\succ\!\succ$ (i.e. $n+m>0$, for $n$ and $m$ at lines 2 and 3 respectively), otherwise, since no preferences over the arguments exist, the base scores cannot be dictated from them. First, the algorithm assigns a distance $d=1$ to the most preferred arguments. For each subsequent argument in the ordering, the assigned distance depends on the preference relation with its predecessor: if the arguments are equally preferred ($\simeq$), the distance remains unchanged. If the relationship between arguments is of strict preference ($\succ$), the distance is increased by $\delta$, and if the relation is of much preferred ($\succ\!\succ$), the distance is increased by $\Delta$.
    %

    \begin{defn}\label{def:algorithm_distance_constants}
        The distance between preferred arguments is $\delta\in(0,\infty)$, and for much more preferred arguments is $\Delta\in(0,\infty)$.
    \end{defn}
    
    After assigning this distance, the base scores are obtained using one of the proposed BSEFs. An example applying these distances in Alg.~\ref{alg:flexible_base_score_extraction} is shown: 

    \begin{algorithm}[t]
    \caption{Flexible Normalised Base Score Extraction}
    \label{alg:flexible_base_score_extraction}
    \begin{algorithmic}[1]
        \Require A set of arguments $\mathcal{X}$, its preference ordering $\succeq$, 
        the increase for greater preference relation $\delta$, for much greater preference relation $\Delta$, 
        the BSEF $\nu_i
        $ (one of $\nu_1$ or $\nu_2$).
        \Ensure Set of base scores $\mathcal{T}$ for the arguments in $\mathcal{X}$
        \State $\mathcal{T} \gets \{\}$ \Comment{create an empty set of base scores}
        \State $n \gets$ number of greater preferences ($\succ$) in $\succeq$
        \State $m \gets$ number of much greater preferences ($\succ\!\succ$) in $\succeq$
        \State $D \gets 1 + n \cdot \delta + m \cdot \Delta$
        \State $d \gets 1$ \Comment{set the first distance value}
        \For{each argument $a \in \mathcal{X}$ ordered by $\succeq$}
            \If{$a$ is not most preferred}
                \State $b \gets$ the previous argument in $\succeq$
                \If{$b \succ a$}
                    \State $d(a) \gets d(a) + \delta$
                \ElsIf{$b \succ\!\succ a$}
                    \State $d(a) \gets d(a) + \Delta$
                \Else
                    \State $d(a) \gets d(a)$ 
                \EndIf
            \EndIf
            \State $\tau(a) \gets \nu_i(
            \succeq)(a)$ \Comment{Apply Eq.~\ref{bsef:1} for $i=1$ or Eq.~\ref{bsef:2} for $i=2$}
            \State $\mathcal{T} \gets \mathcal{T} \cup \{\tau(a)$\}
        \EndFor
        \State \Return $\mathcal{T}$
    \end{algorithmic}
\end{algorithm}

    \begin{example}
    \label{ex:algo}
        The robot feeds a user with the following preference ordering: $c\simeq f \succ\!\succ b\simeq e \succ a \simeq d$. It confirms that the ratio for the \emph{much greater} preference is 3. Then, the function parameters are set to $\delta=1, \Delta=3$. By choice, the limits are set to $\top=0.8, \bot=0.2$, and  Property~\ref{property:regularity} is satisfied. The distances are $d(c)=d(f)=1$, $d(b)=d(e)=4$, and $d(a)=d(d)=5$.\footnote{$d(d)$ refers to the distance of argument $d$.} Since the limits are arbitrarily set, the BSEF to be used is $\nu_{\ref{eq:base_score_function_margins}}$, and the base scores result in: $\tau(c)=\tau(f)=0.8, \tau(b)=\tau(e)=0.35, \tau(a)=\tau(d)=0.2$. 
    \end{example}

    The base scores obtained can be used for decision-making, as shown in the following example.

    \begin{example}
        (Example~\ref{ex:algo} Cont.) The final strengths for the decision arguments are computed based on the previous example base scores, using the QE Model semantics, which results in: $\sigma(D_1)=0.54$ and $\sigma(D_2)=0.40$. Then, the robot decides to move slowly. Note that the user prefers arguments $c$ and $f$, which increase the strength of the slow option and decrease the strength of the fast option, respectively. The arguments $b$ and $e$ increase and decrease the strength of the slow option, respectively. Finally, the arguments $a$ and $d$  aim to make the robot go fast. The algorithm's selection of moving the robot slowly accurately adapts to the user's preferences. 
    \end{example}




    These BSEFs offer different benefits. First, both are flexible, allowing the base scores to be set in a desired range and proportionally adjusting the difference of base scores between arguments. Additionally, the parameters are easily modifiable, allowing corrections and learning from human feedback. Other functions could be used, e.g., an exponentially increasing or decreasing function, but there is a high chance that they do not adjust properly to a user's preferences. For example, given $a,b,c,d,e\in\mathcal{X}$ ordered as: $a\succ b\succ c\succ d\succ e$, an exponentially decreasing function could set the base scores at $\tau(a)=1$, $\tau(b)=0.9$, $\tau(c)=0.75$, $\tau(d)=0.5$, and $\tau(e)=0.1$. Assuming that the distribution of the base scores follows that function is most probably incorrect. Allowing the user to confirm the ordering as: $a\succ b \succ c\succ d\succ\!\succ e$, the base scores could be set to $\tau(a)=1$, $\tau(b)=0.866$, $\tau(c)=0.677$, $\tau(d)=0.5$, $\tau(e)=0.1$. These base scores are similar to the previous ones, but the confirmation of the large gap between arguments $d$ and $e$ provides a more accurate representation of the real preferences.

\section{Evaluation}

     This section presents a theoretical validation of the concrete BSEFs based on how they satisfy the theoretical properties we defined. We also undertake an experimental evaluation, showing the pros and cons of each BSEF. Finally, we give an intuition of which gradual semantics to select depending on the decision-making context.


    \subsection{Theoretical Analysis}

        First, we show under which conditions our proposed BSEFs satisfy the axioms and properties presented. This is summarised in Table~\ref{tab:axioms_properties}. The proofs can be found in~\cite{Civit_arxiv26}. 

        \begin{table}[t]
        \centering
        \resizebox{\columnwidth}{!}{%
        \begin{tabular}{l l l l l l l l}
        \hline
        $\nu$ & A\ref{axiom:1} & A\ref{axiom:2} & A\ref{axiom:3} & P\ref{property:1} & P\ref{property:2} & P\ref{property:3} & P\ref{property:4} \\ \hline
        $\nu_{\ref{bsef:1}}$ & \makecell{$\Delta>0$ \\ $\delta>0$} & $\Delta>\delta>0$ & \cmark & \makecell{$\top=1$ \\ $\bot=0$} & $\top=1-\bot$ & \makecell{$\Delta$ and $\delta$ \\ constants} & \cmark \\ \hline
        $\nu_{\ref{bsef:2}}$ & \makecell{$\Delta>0$ \\ $\delta>0$} & $\Delta>\delta>0$ & \cmark & $\alpha=\beta=0$ & $\alpha=\beta/2$ & \makecell{$\Delta$ and $\delta$ \\ constants} & \cmark \\ \hline
        \end{tabular}
        }
        \caption{Summary of the conditions for satisfying the presented axioms (A) and desirable properties (P).}
        \label{tab:axioms_properties}
        \end{table}

        For Axiom 1, constants $\Delta$ and $\delta$ must be greater than zero. 
        
        \begin{proposition}\label{prop:axiom_1}
            Setting the constants $\Delta>0$ and $\delta>0$ satisfies Axiom 1 for both BSEFs.
        \end{proposition}


        \begin{table*}[t]
        \footnotesize
        \centering
        \resizebox{0.85\linewidth}{!}{%
        \begin{tabular}{|c|lll||ll||ll||ll|}
            \hline
            \multicolumn{1}{|l|}{\multirow{2}{*}{Preference ordering $\succeq$}} & \multicolumn{3}{l||}{Design Choices}                                         & \multicolumn{2}{l||}{QE Strengths}                    & \multicolumn{2}{l||}{EB Strengths}                               & \multicolumn{2}{l|}{DF Strengths}                                       \\ \cline{2-10} 
            \multicolumn{1}{|l|}{}                                                 & \multicolumn{1}{l|}{$\top$} & \multicolumn{1}{l|}{$\bot$} & $\Delta/\delta$ & \multicolumn{1}{l|}{$\sigma_{slow}$} & $\sigma_{fast}$ & \multicolumn{1}{l|}{$\sigma_{slow}$} & $\sigma_{fast}$ & \multicolumn{1}{l|}{$\sigma_{slow}$} & $\sigma_{fast}$ \\ \hline
            $a\simeq b\simeq c \simeq d \simeq e \simeq f$                         & \multicolumn{1}{l|}{\text{\Large -}}      & \multicolumn{1}{l|}{\text{\Large -}}      & $\text{\Large -}$               & \multicolumn{1}{l|}{0.5}                  & \textbf{0.51}        & \multicolumn{1}{l|}{0.50}                 & \textbf{0.52}       & \multicolumn{1}{l|}{0.44}                      & \textbf{0.63}             \\ \hline
            \multirow{3}{*}{$c\simeq f \succ\!\succ b\simeq e \succ a \simeq d$}   & \multicolumn{1}{l|}{0.9}    & \multicolumn{1}{l|}{0.1}    & 3               & \multicolumn{1}{l|}{\textbf{0.58}}        & 0.33                 & \multicolumn{1}{l|}{\textbf{0.56}}        & 0.39                 & \multicolumn{1}{l|}{\textbf{0.78}}             & 0.23                      \\ \cline{2-10} 
                                                                                   & \multicolumn{1}{l|}{0.9}    & \multicolumn{1}{l|}{0.1}    & 5               & \multicolumn{1}{l|}{\textbf{0.57}}        & 0.32                 & \multicolumn{1}{l|}{\textbf{0.55}}        & 0.39                 & \multicolumn{1}{l|}{\textbf{0.79}}             & 0.2                       \\ \cline{2-10} 
                                                                                   & \multicolumn{1}{l|}{0.75}   & \multicolumn{1}{l|}{0.25}   & 5               & \multicolumn{1}{l|}{\textbf{0.52}}        & 0.4                  & \multicolumn{1}{l|}{\textbf{0.53}}        & 0.43                 & \multicolumn{1}{l|}{\textbf{0.63}}             & 0.38                      \\ \hline
            \multirow{2}{*}{$b\simeq e \succ a\simeq d \succ\!\succ c \simeq f$}   & \multicolumn{1}{l|}{0.8}    & \multicolumn{1}{l|}{0.2}    & 3               & \multicolumn{1}{l|}{0.5}                  & \textbf{0.63}        & \multicolumn{1}{l|}{0.52}                 & \textbf{0.60}        & \multicolumn{1}{l|}{0.28}                      & \textbf{0.87}             \\ \cline{2-10} 
                                                                                   & \multicolumn{1}{l|}{0.6}    & \multicolumn{1}{l|}{0.4}    & 3               & \multicolumn{1}{l|}{0.5}                  & \textbf{0.53}        & \multicolumn{1}{l|}{0.5}                  & \textbf{0.54}        & \multicolumn{1}{l|}{0.40}                      & \textbf{0.71}             \\ \hline
            \multirow{2}{*}{$a\simeq d \succ\!\succ c\simeq f \succ b \simeq e$}   & \multicolumn{1}{l|}{0.8}    & \multicolumn{1}{l|}{0.2}    & 4               & \multicolumn{1}{l|}{0.36}                 & \textbf{0.6}         & \multicolumn{1}{l|}{0.4}                  & \textbf{0.59}        & \multicolumn{1}{l|}{0.15}                      & \textbf{0.76}             \\ \cline{2-10} 
                                                                                   & \multicolumn{1}{l|}{1}      & \multicolumn{1}{l|}{0}      & 4               & \multicolumn{1}{l|}{0.25}                 & \textbf{0.7}         & \multicolumn{1}{l|}{0.37}                 & \textbf{0.65}        & \multicolumn{1}{l|}{0}                         & \textbf{0.9}              \\ \hline
        \end{tabular}}
        \caption{Examples of preference orderings for the running example with different design choices for the BSEF $\nu_1$ (Eq.~\ref{bsef:1}). The value in bold between $\sigma_{slow}$ and $\sigma_{fast}$ is the system output. The (\text{\Large -}) symbol indicates when a design choice is not needed, which is the case when all arguments are equally preferred, and their base scores are set to 0.5.}
         \label{table:examples_prefs_choices_margins}
        \end{table*}

        For Axiom 2, it is necessary that the distance between much more preferred arguments is bigger than the distance between preferred arguments, and must be greater than zero.

        \begin{proposition}\label{prop:axiom_2}
            Setting $\Delta>\delta>0$ satisfies Axiom 2 for both BSEFs.
        \end{proposition}


        Since the BSEFs are based on the distance between arguments given an ordering, if that ordering is changed, the base scores' ordering will also change. 

        \begin{proposition}\label{prop:axiom_3}
            Both BSEFs satisfy Axiom 3. 
        \end{proposition}


        Next, we focus on the BSEFs' desirable properties. First, on the range properties. First, in $\nu_{\ref{eq:base_score_function_margins}}$.
    
        \begin{proposition}\label{prop:3}
            Setting $\top=1$ and $\bot=0$ in $\nu_{\ref{bsef:1}}$ satisfies \textnormal{Base Score Normalisation} (Property~\ref{property:normalisation}).
        \end{proposition}
    
        \begin{proposition}\label{prop:4}
            Setting $\top=1-\bot$ (with $\bot \leq 0.5$) in $\nu_{\ref{bsef:1}}$ satisfies \textnormal{Base Score Centralisation} (Property~\ref{property:centralisation}).
        \end{proposition}
    
    
        Now, for the other BSEF $\nu_{\ref{eq:base_score_function_ranking_a_b}}$, those properties can be satisfied by adjusting the parameters $\alpha$ and $\beta$. 
            
        \begin{proposition}\label{prop:5}
            Setting $\alpha=\beta=0$ in $\nu_{\ref{bsef:2}}$ satisfies \textnormal{Base Score Normalisation} (Property~\ref{property:normalisation}).
        \end{proposition}
    
    
        \begin{proposition}\label{prop:6}
            Setting $\alpha=\frac{\beta}{2}$ in $\nu_{\ref{bsef:2}}$ leads to equal limits $\top=1-\bot$, satisfying \textnormal{Base Score Centralisation} (Property~\ref{property:centralisation}).
        \end{proposition}
    
    
        To satisfy Property~\ref{property:regularity} (Base Score Relation Regularity), the algorithm variables to compute the distance ($\Delta, \delta$) must be constant. 
    
        \begin{proposition}\label{prop:7}
            Setting $\Delta$ and $\delta$ as fixed constants satisfies \textnormal{Base Score Relation Regularity} (Property~\ref{property:regularity}) for $\nu_{\ref{bsef:1}}$ and $\nu_{\ref{bsef:2}}$. 
        \end{proposition}
    

         Property \ref{property:base_score_preference_stability} determines that the base scores of a set of arguments with a given preference relation remain constant when adding a new argument that is equally preferred to an existing one. 
    
        \begin{proposition}\label{prop:8}
            Both BSEFs satisfy \textnormal{Base Score Preference Stability} (Property~\ref{property:base_score_preference_stability}).
        \end{proposition}
    

        Finally, if the selected gradual semantics is monotonic and balanced~\cite{Baroni_IJAR19}, if an argument is preferred to another and both have the same influences, the preferred argument will be stronger.
    
        \begin{proposition}\label{prop:2}
            Given a BSEF $\nu$, where $\nu$ satisfies Axiom 1, and a semantics $\sigma$, where $\sigma$ satisfies monotonicity and balance, for any arguments $a$, $b$ where $a \succ b$ and $a$ and $b$ have equivalent supporters and attackers, then $\sigma(a) > \sigma(b)$.
        \end{proposition}
    

    \subsection{Experimentation}

        The aim of this experiment is to evaluate the practical behaviour of the proposed BSEFs under different user preference profiles and design choices, and to examine how various gradual semantics influence the resulting decisions. Specifically, we test whether the generated QBAFs lead to consistent, interpretable, and preference-aligned outcomes across semantics.

        We consider the running example from Fig.~\ref{fig:intro_vertical} in which the robot must choose between slow ($D_1$) and fast ($D_2$) feeding paces (see Example~\ref{ex:introduction}). Each QBAF instance is constructed using six arguments and a pair of decision arguments.

        To explore the robustness of our approach, we generated 30,000 random samples of preference orderings and design choices, recording the option selected by three gradual semantics. These semantics are: QE Model,  Euler-Based (EB)~\cite{Amgoud_IJAR18}, and  DF-QuAD (DF)~\cite{Rago_kr16}, which are among the most used in the literature. The decision arguments' base scores are set to 0.5. We computed pairwise agreement and Cohen’s Kappa coefficient ($\kappa$) to assess consistency between 
        semantics. Under the centralisation property, the pairwise agreement values were $QE-EB = 0.98$, $QE-DF = 0.85$, and $EB-DF = 0.84$. The corresponding Cohen’s Kappa scores were $QE-EB = 0.96$, $QE-DF = 0.65$, and $EB-DF = 0.62$. Without centralisation, the pairwise agreement values were $QE-EB = 0.96$, $QE-DF = 0.81$, and $EB-DF = 0.81$, with Cohen’s Kappa scores of $QE-EB = 0.92$, $QE-DF = 0.52$, and $EB-DF = 0.52$. 
        These results show a strong alignment between QE and EB, while DF-QuAD diverges more, particularly when base score ranges are extreme (e.g., $\tau(e) \approx 1$, $\tau(e) \approx 0$).
        This difference shows the structural properties of the methods discussed in Sec.~\ref{sec:select_model}. \looseness=-1

        To further analyse the influence of user preferences and design parameters, Table~\ref{table:examples_prefs_choices_margins} reports examples of preference orderings combined with different design choices, evaluated under the three gradual semantics. The Base Score Relation Regularity property is assumed to reduce experimental complexity and noise.

        The results confirm that the preference ordering itself is the dominant determinant of the final decision, while design choices act as fine-tuning factors that adjust numerical sensitivity but sometimes invert the selected option.
        For instance, when arguments related to safety ($c$ and $f$) are prioritised over those concerning enjoyment ($a$ and $d$), all semantics consistently select the slow feeding pace, aligning with the user’s intended caution.
        Conversely, when enjoyment-related arguments are most preferred, the system reliably opts for fast feeding, demonstrating the coherence of the extracted base scores with user values.
        
        Design choices such as the range limits affect the magnitude of the option strengths rather than the decision direction. Narrower ranges (e.g., $(\top, \bot) = (0.75, 0.25)$) compress the difference between argument strengths, producing less decisive but more stable outcomes, while broader ranges (e.g., $(1, 0)$) amplify contrasts and make the system more sensitive to preference changes.
        Similarly, the preference ratio $(\Delta / \delta)$ controls how much stronger a \emph{much greater} preference influences the base scores. Increasing this ratio enhances differentiation between preference tiers but may also overemphasise extreme preferences, especially under DF-QuAD, which is more responsive to high base score values.

    \subsection{Discussion on the Gradual Semantics}
    \label{sec:select_model}
    
        The previous analysis revealed that, although all semantics preserve preference–decision coherence, their sensitivity to base score variations differs substantially.
        This behaviour stems from the aggregation and influence functions each model employs.

        \begin{figure*}
            \centering
            \includegraphics[width=0.99\linewidth]{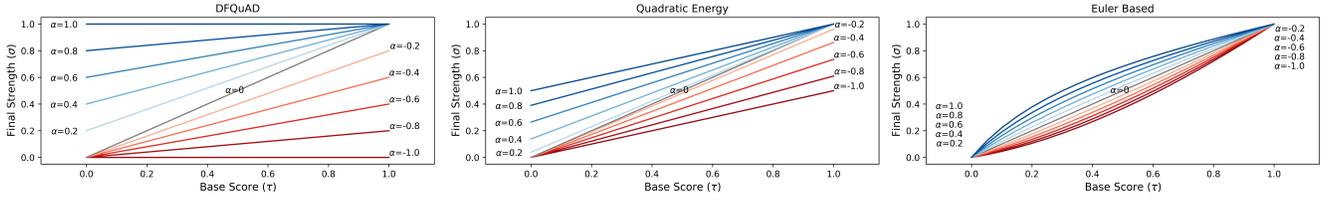}
            \caption{Influence of an argument’s base score and its supporters or attackers on the final argument strength under different gradual semantics.
            Each subplot corresponds to a distinct semantics: DF-QuAD (left), Quadratic Energy  (centre), and the Euler-based Semantics (right). The x-axis represents the base score of the influenced argument, and the y-axis its resulting final strength. Blue lines denote when the influencing argument is a supporter (+), while red lines denote when it is an attacker (-). Each line has assigned its aggregation ($\alpha$), with colour intensity indicating its strength.}
            \label{fig:influences}
            \Description{Each subplot shows the influence of an argument's base score (x-axis), and the different aggregation strengths (lines) for that argument's final strength (y-axis) for the different gradual semantic models. The left figure is for DF-QuAD, the centre one is for Quadratic Energy, and the right is for Euler-based. The blue lines are for supports, while the red lines are for attacks. At the tip of the lines is written the base score of the supporter (+) or attacker (-). The colour intensity of the lines represents the strength of the attack or support.}
        \end{figure*}

        Figure~\ref{fig:influences} shows the influence of a single attacker or supporter (the influencer), and an argument's base score (the influenced) for its final strength. The DF-QuAD semantics use product aggregation with a linear influence function, which causes the influencer’s base score to dominate the final outcome.
        When an attacking (or supporting) argument has a base score close to 1, it almost completely drives the influenced argument’s strength toward 0 (or 1), regardless of the influenced base score. This strong reactivity can be desirable in safety-critical scenarios, where highly preferred arguments should decisively determine the outcome, but it may reduce robustness in settings with uncertain or noisy preferences.
        In contrast, the QE semantics combines sum aggregation with a 2-Max influence function. Here, both the influencer and the influenced arguments contribute comparably to the final strength, leading to smoother transitions and more gradual adaptation to preference changes (e.g., given an influencer with a base score of 1, the influenced's final strength will be $\sigma=0.5\tau+0.5$ if it receives a support, and $\sigma=0.5\tau$ if it receives an attack). 
        QE, therefore, offers a good compromise between responsiveness and stability, which makes it particularly suitable for explainable decision-support systems that require consistent reasoning traces. Finally, the EB semantics is the most conservative, using sum aggregation with an Euler-based influence function. When an argument’s base score is close to the extremes (0 or 1), its final strength remains nearly fixed at those values, regardless of new supports or attacks. This behaviour preserves high-confidence arguments, making EB appropriate when users insist on arguments being non-negotiable or when decisions must respect rigid value priorities.\looseness=-1

\section{Related Work}

    Gradual argumentation relies on the arguments' base scores, and its literature has unveiled different ways of determining them depending on the context. In~\cite{Baroni_AC15}, the QuAD framework was introduced to support debates on design alternatives. Since these scenarios do not involve a large number of voters, base scores are assessed by experts or derived from predefined criteria. As an extension, in~\cite{Rago_icppmas17}, the authors presented QuAD-V, which allows users to vote for or against arguments, using these votes to compute base scores. A similar procedure is carried out in~\cite{Tarle_aamas22}. In another work~\cite{Cocarascu_aamas19} in which a QBAF is generated from movie reviews, the base scores are obtained from an aggregation of critics' votes. 
    Another use of QBAFs is to apply their methods for explainable AI, e.g. by transforming (deep) neural networks into QBAFs~\cite{Albini_arxiv20,DBLP:conf/ecai/AyoobiPT23}. For instance, in~\cite{Albini_arxiv20}, the base scores of the arguments are obtained from the attacking, supporting neurons, and their activations in the network, similarly to~\cite{Potyka_aaai21}. In contrast, the work from~\cite{Battaglia_icsum24} sets base scores according to users' preferences to personalise decision support, similarly to our work. There, base scores are initially set arbitrarily and then discounted for less preferred arguments. This work similarly relies on preferences, but uses them to determine the initial base scores directly.\looseness=-1

    Within structured argumentation, the work in~\cite{Tamani_icfs14} focuses on ASPIC argumentation frameworks with fuzzy set theory, using expert knowledge to assess argument importance, similarly to a base score.

    In summary, two main limitations are found in existing work. First, most approaches do not account for the preferences of a single user when setting base scores, limiting personalisation in argumentation-based decision-making systems. Second, when preferences are considered, base scores are initialised arbitrarily and adjusted afterwards. 
    In this work, we address those research gaps. 

\section{Conclusions and Future Work}

    
    This work introduces BSEFs, a principled method to derive argument base scores from user preference orderings in gradual argumentation. 
    We defined a set of axioms and desirable properties, proposed two concrete BSEFs with tunable design choices, and incorporated non-linear preference intensity to better capture human reasoning.
    Theoretical analysis confirmed that the functions satisfy coherence and regularity properties, while experiments in a robotics scenario showed that preference ordering dominated the decision outcome and design choices primarily modulate sensitivity.


    Future directions include the fine-tuning or definition of the different design choices given the context and user feedback, allowing partial orderings in preferences in case some arguments cannot be compared~\cite{Visser_wtafa11}, and computing the aggregation when using this approach when considering more than one user's preferences in the decision-making~\cite{Rago_AI25}. A philosophical question which should be commented on is whether it is realistic to consider the scenario in which the most preferred argument supports the least preferred argument. 
    Our focus was on bipolar argumentation, but gradual structured argumentation frameworks are receiving attention lately~\cite{Rago_kr25, Rapberger_kr25}. A future extension of this work is to adapt our method to structured argumentation. Furthermore, we plan to validate our approach by performing a user study. Finally, the extension of this work into value-based argumentation could provide an approach for more value-aligned decision-making~\cite{Atkinson_AL21, Bodanza_AC23}. This could be potentially achieved by ordering the values promoted by the arguments and assigning a base score to each value.

\begin{acks}

This work was supported by the project CHLOE-MAP PID2023-152259OB-I00 funded by MCIU/ AEI /10.13039/501100011033 and by ERDF, UE; the project ROBOCAT SDC006/25/000016 funded by the Generalitat de Catalunya, NextGenerationEU.
A. Civit has been supported by AGAUR-FI ajuts (2023 FI-3 00065) Joan Oró of the 
Generalitat of Catalonia and the European Social Plus Fund.
F. Toni has been supported by the European Research Council (ERC) under the European
Union’s Horizon 2020 research and innovation programme (grant agreement No. 101020934).
\end{acks}

\bibliographystyle{ACM-Reference-Format} 
\bibliography{references}



\newpage

\begin{appendices}
\section{Proofs}
\label{sec:proofs}

Here we give the proofs for the theoretical work in the paper.
\\
\\
Proof for Proposition~\ref{prop:1}:

\begin{proof}
    Let $a,b\in\mathcal{X}$, where $a \succ b$ in the given descending preference ordering, and $f(x)$ is a monotonically increasing function. By definition, $f(x)\leq f(x+\Delta x)$. Using Definition~\ref{def:algorithm_distance_constants}, if $a\succ b$ the distances from $a$ ($d(a)$) and $b$ ($d(b)$) satisfy $d(a)<d(b)$, therefore, the BSEF~\ref{bsef:1}, we obtain $\nu_{\ref{bsef:1}}(\succeq)(a)>\nu_{\ref{bsef:1}}(\succeq)(b)$, since the variable part of the function $(D-d(x))/(D-1)$ is lower when $d(x)$ is greater. The same occurs for BSEF~\ref{bsef:2}.
\end{proof}

Proof for Proposition~\ref{prop:axiom_1}:

\begin{proof}
    Let $a,b\in\mathcal{X}$ be two adjacent arguments in the ordering, such that $a \succ b$ in the preference ordering. According to Definition~\ref{def:algorithm_distance_constants} and Algorithm~\ref{alg:flexible_base_score_extraction}, the distance assigned to $a$ and $b$ satisfies $d(b) = d(a) + \delta$. Where $\delta>0$ for a regular preference ($\succ$) and $\Delta$>0 for a \emph{much greater} preference ($\succ\!\succ$). Both proposed base score extraction functions $\nu_1$ and $\nu_2$ are strictly decreasing with respect to $d(x)$. Therefore, if $d(a) < d(b)$, it follows that 
$\nu(\succeq)(a) > \nu(\succeq)(b)$. 
    This implies that the base score of the preferred argument $a$ is greater than that of $b$, i.e., $\tau(a) > \tau(b)$.
    Hence, setting $\delta > 0$ (and similarly $\Delta > 0$) ensures that both base score extraction functions satisfy Axiom~\ref{axiom:1}.
\end{proof}
Proof for Proposition~\ref{prop:axiom_2}:
\begin{proof}
    Let $a,b,c,d\in\mathcal{X}$ with a preference ordering where $a \succ\!\succ b \succ c \simeq d$. According to Definition~\ref{def:algorithm_distance_constants} and Algorithm~\ref{alg:flexible_base_score_extraction}, the distances assigned to these arguments satisfy
    \[
    d(b) = d(a) + \Delta, \qquad d(c) = d(b) + \delta, \qquad d(d) = d(c),
    \]
    where $\Delta$ and $\delta$ are the increments for the relations $\succ\!\succ$ and $\succ$, respectively.

    For both proposed base score extraction functions, $\nu_1$ and $\nu_2$, the base score difference between two consecutive arguments is proportional to their distance difference.
    Hence,
    \[
    \tau(a) - \tau(b) \propto \Delta, \qquad
    \tau(b) - \tau(c) \propto \delta, \qquad
    \tau(c) - \tau(d) = 0.
    \]
    
    If $\Delta > \delta > 0$, it follows that
    \[
    \tau(a) - \tau(b) > \tau(b) - \tau(c) > \tau(c) - \tau(d),
    \]
    which satisfies Axiom~\ref{axiom:2}.
    Therefore, setting $\Delta > \delta > 0$ guarantees that both $\nu_1$ and $\nu_2$ respect the intended magnitude differences between \emph{much greater} and \emph{greater} preferences.
\end{proof}
Proof for Proposition~\ref{prop:axiom_3}:
\begin{proof}
    Let $\nu(\succeq)$ be a base score extraction function applied to a set of arguments $X$
    with a preference ordering $\succeq$.
    By definition, $\nu(\succeq)$ assigns base scores according to the relative
    ordering and distances among the arguments.
    
    Suppose now that the preference ordering changes to a non-isomorphic
    ordering $\succeq'$.
    By Definition~\ref{def:pref_ordering}, two preference orderings are isomorphic if and only if
    there exists a bijection $f : X \to X'$ such that, for all $i,j \in X$,
    \[
    i \succeq j \;\Leftrightarrow\; f(i) \succeq' f(j).
    \]
    If $\succeq$ and $\succeq'$ are \emph{not} isomorphic,
    then at least one pair of arguments $(i,j)$ exists such that their preference
    relations differ, i.e.,
    \[
    i \succ j \quad \text{but} \quad f(i) \not\succ' f(j),
    \]
    or vice versa.
    
    Because both proposed base score extraction functions
    $\nu_1$ and $\nu_2$ are monotonic with respect to the preference order,
    a change in the ordering of any pair of arguments necessarily alters
    their relative distances $d(i)$, $d(j)$ and hence their base scores
    $\tau(i)$ and $\tau(j)$.
    Consequently,
    \[
    \nu(\succeq) \neq \nu(\succeq'),
    \]
    which satisfies Axiom~\ref{axiom:3}.
\end{proof}
Proof for Proposition~\ref{prop:3}:
\begin{proof}
    Given the definition of $\nu_1$, if we set $\top = 1$ and $\bot = 0$, then
    \[
    \nu_1(\succeq)(x) \;=\; 0 + (1-0)\,\frac{D-d(x)}{D-1}
    \;=\; \frac{D-d(x)}{D-1}.
    \]
    For the most preferred arguments $M\in\mathcal{X}$, we have $\forall m\in M, \ d(m)=1$, hence
    \[
    \nu_1(\succeq)(m) \;=\; \frac{D-1}{D-1} \;=\; 1.
    \]
    For the least preferred arguments $L\in\mathcal{X}$, we have $\forall l\in L, \ d(l)=D$, hence
    \[
    \nu_1(\succeq)(l) \;=\; \frac{D-D}{D-1} \;=\; 0.
    \]
    Because $d(x) \in \{1,\dots,D\}$ and the function $(D-d(x))/(D-1)$ decreases linearly from $1$ to $0$ as $d(x)$ goes from $1$ to $D$, all base scores lie in $[0,1]$ with the required endpoints. Therefore, $\nu_1$ with $\top=1$ and $\bot=0$ satisfies the Base Score Normalisation property.
\end{proof}
Proof for Proposition~\ref{prop:4}:
\begin{proof}
    Given the definition of $\nu_1$, if we set $\top = 1 - \bot$ with $\bot \leq 0.5$, then
    \[
    \top - \bot \;=\; (1-\bot)-\bot \;=\; 1 - 2\bot.
    \]
    Substituting into $\nu_1$ yields
    \[
    \nu_1(\succeq)(x) \;=\; \bot + (1-2\bot)\,\frac{D-d(x)}{D-1}.
    \]
    For the most preferred arguments ($m\in\mathcal{X}$ such that $d(m)=1$),
    \[
    \nu_1(\succeq)(m) \;=\; \bot + (1-2\bot)\,\frac{D-1}{D-1}
    = \bot + (1-2\bot) = 1-\bot = \top.
    \]
    For the least preferred arguments ($l\in\mathcal{X}$ such that $d=D$),
    \[
    \nu_1(\succeq)(l) \;=\; \bot + (1-2\bot)\,\frac{D-D}{D-1}
    = \bot.
    \]
    Moreover, the average of the endpoint values is
    \[
    \frac{\top + \bot}{2} \;=\; \frac{(1-\bot) + \bot}{2} \;=\; \frac{1}{2},
    \]
    so the endpoints are symmetric about $0.5$. Since $\bot \leq 0.5$ guarantees $\top \geq \bot$, the assignment is valid within $[0,1]$. Therefore, setting $\top = 1-\bot$ (with $\bot \leq 0.5$) makes $\nu_1$ satisfy the Base Score Centralisation property.
\end{proof}
Proof for Proposition~\ref{prop:5}:
\begin{proof}
    Given the definition of $\nu_2$, if we set $\alpha=\beta=0$, then 
    \[
    \nu_2(\succeq)(x)\;=\; \frac{D-d(x)}{D-1}.
    \]
    For the most preferred arguments $m\in\mathcal{X}$, we have $d(m)=1$, hence
    \[
    \nu_2(\succeq)(m) \;=\; \frac{D-1}{D-1} \;=\; 1.
    \]
    For the least preferred arguments $l\in\mathcal{X}$, we have $d(l)=D$, hence
    \[
    \nu_2(\succeq)(l) \;=\; \frac{D-D}{D-1} \;=\; 0.
    \]
    Since $d(x)\in\{1,\dots,D\}$, the value $(D-d(x))/(D-1)$ ranges linearly from $1$ to $0$ as $d(x)$ goes from $1$ to $D$, and all values lie in $[0,1]$. Therefore, $\nu_2$ with $\alpha=\beta=0$ satisfies the Base Score Normalisation property.
\end{proof}
Proof for Proposition~\ref{prop:6}:
\begin{proof}
    Given the definition of $\nu_2$, for the most preferred arguments $m\in\mathcal{X}$ ($d(m)=1$) we obtain the upper limit
    \[
    \top \;=\; \nu_2(\succeq)(m)=\frac{D-1+\alpha}{D-1+\beta},
    \]
    and for the least preferred arguments $l\in\mathcal{X}$ ($d(l)=D$), the lower limit is
    \[
    \bot \;=\; \nu_2(\succeq)(l)=\frac{\alpha}{D-1+\beta}.
    \]
    We require $\top = 1-\bot$. Substituting the expressions above and bringing them to a common denominator gives
    \[
    \frac{D-1+\alpha}{D-1+\beta} \;=\; 1 - \frac{\alpha}{D-1+\beta}
    \;=\; \frac{D-1+\beta-\alpha}{D-1+\beta}.
    \]
    Equating numerators yields
    \[
    D-1+\alpha \;=\; D-1+\beta-\alpha \quad\Longrightarrow\quad 2\alpha=\beta,
    \]
    i.e. $\alpha=\beta/2$. Hence, setting $\alpha=\beta/2$ ensures $\top=1-\bot$, so $\nu_2$ satisfies the Base Score Centralisation property.
\end{proof}
Proof for Proposition~\ref{prop:7}:
\begin{proof}
    Let $(a,b)$ and $(c,d)$ be two pairs of adjacent arguments in a preference ordering,
    and suppose both pairs share the same preference relation type
    (i.e.\ both are $\succ$, both are $\succ\!\succ$, or both are $\simeq$).
    By Algorithm~\ref{alg:flexible_base_score_extraction}, the distance difference between the members of each adjacent pair is identical to the increment associated with that relation:
    \[
    d(b) - d(a) = d(d) - d(c) =
    \begin{cases}
    \delta & \text{if the relation is }\succ,\\[4pt]
    \Delta & \text{if the relation is }\succ\!\succ,\\[4pt]
    0 & \text{if the relation is }\simeq.
    \end{cases}
    \]
    
    We show that for both proposed base score extraction functions, the base score difference between the elements of an adjacent pair depends only on this distance difference (hence, it is the same for both pairs).
    
    \medskip
    
    \noindent\emph{For \(\nu_1\):}
    \[
    \nu_1(\succeq)(x)=\bot+(\top-\bot)\frac{D-d(x)}{D-1}.
    \]
    Thus, for an adjacent pair \((x,y)\) we have
    \begin{align*}
        \tau(x)-\tau(y)
        &= \nu_1(\succeq)(x)-\nu_1(\succeq)(y) \\
        &= (\top-\bot)\frac{(D-d(x))-(D-d(y))}{D-1} \\
        &= (\top-\bot)\frac{d(y)-d(x)}{D-1}.
    \end{align*}
    Hence if \(d(y)-d(x)\) is the same (e.g.\ \(\delta\) or \(\Delta\)) for both pairs,
    then \(\tau(a)-\tau(b)=\tau(c)-\tau(d)\). For the indifference relation (\(\simeq\)),
    \(d(y)-d(x)=0\) and the difference is \(0\).
    
    \medskip
    
    \noindent\emph{For \(\nu_2\):}
    \[
    \nu_2(\succeq)(x)=\frac{D-d(x)+\alpha}{D-1+\beta}.
    \]
    For an adjacent pair \((x,y)\) we obtain
    \[
    \tau(x)-\tau(y)
    = \frac{D-d(x)+\alpha}{D-1+\beta}-\frac{D-d(y)+\alpha}{D-1+\beta}
    = \frac{d(y)-d(x)}{D-1+\beta}.
    \]
    Again, this difference depends only on \(d(y)-d(x)\); therefore, equal increments \(\delta\) or \(\Delta\) produce equal base score differences for the two pairs, and indifference yields zero difference.
    
    Since fixing \(\delta\) and \(\Delta\) makes the distance differences between any two adjacent arguments with the same relation identical, both \(\nu_1\) and \(\nu_2\) assign equal base score gaps to such adjacent pairs. This is precisely the statement of Base Score Relation Regularity.
\end{proof}
Proof for Proposition~\ref{prop:8}:
\begin{proof}
    Let $\succeq$ be a preference ordering over $\mathcal{X}$ and let $S$ be a set of new arguments such that
    $\mathcal{X}' = \mathcal{X} \cup S$ and for every $s\in S$ there exists $a\in \mathcal{X}$ with $a \simeq s$ under the extended ordering $\succeq'$.
    By construction, adding $S$ introduces only \emph{indifference} relations ($\simeq$) between new arguments and some existing arguments in $\mathcal{X}$.
    Observe two points about Algorithm~\ref{alg:flexible_base_score_extraction} and the distance assignment (Definition~\ref{def:algorithm_distance_constants}):
    \begin{enumerate}
      \item The maximum distance span $D = 1 + n\cdot\delta + m\cdot\Delta$ depends only on the counts $n$ and $m$ of $\succ$ and $\succ\!\succ$ relations in the ordering. Since every $s\in S$ is indifferent to some $a\in \mathcal{X}$, the numbers $n$ and $m$ for relations between distinct preference tiers remain unchanged; hence $D$ is unchanged.
      \item When traversing the ordering, Algorithm~\ref{alg:flexible_base_score_extraction} increases the running distance $d(x)$ only when the previous relation was $\succ$ (by $\delta$) or $\succ\!\succ$ (by $\Delta$), and leaves $d(x)$ unchanged when the previous relation is $\simeq$. Therefore, inserting arguments that are indifferent to existing ones does not change the distance values assigned to the original arguments in $\mathcal{X}$ (they either share the same distance as before or receive the same $d(x)$ as their indifferent partner).
    \end{enumerate}
    Because both $\nu_1(\succeq)(x)$ and $\nu_2(\succeq)(x)$ depend only on $D$ and the per-argument distances $d(x)$, and neither $D$ nor the distances of original arguments change under the described insertion of indifferent arguments, it follows that the base scores assigned to arguments in $X$ remain identical:
    \[
    \nu(\succeq) \;=\; \nu'(\succeq').
    \]
    Hence, both $\nu_1$ and $\nu_2$ satisfy Base Score Preference Stability.
\end{proof}
Proof for Proposition~\ref{prop:2}:
\begin{proof}
    Let $\nu$ be a base score extraction function satisfying Axiom~\ref{axiom:1}, and let $\sigma$ be a gradual semantics that satisfies \emph{monotonicity} and \emph{balance}.
    
    Take two arguments $a,b\in X$ such that $a \succ b$ and suppose that $a$ and $b$ have equivalent sets of supporters and attackers (i.e.\ the multiset of incoming supports and attacks and their sources are identical up to renaming between $a$ and $b$).
    By Axiom~\ref{axiom:1}, we have $\tau(a) > \tau(b)$.
    
    Now consider the evaluation of $\sigma$ on the QBAF where all base scores except those of $a$ and $b$ are held fixed.
    Monotonicity of $\sigma$ implies that, when all other inputs are unchanged, increasing an argument's base score (while keeping others fixed) cannot decrease its final strength:
    \[
    \tau(a) > \tau(b) \quad\Longrightarrow\quad \sigma(a) \ge \sigma(b),
    \]
    and more precisely, a strict increase in the base score (with identical influences) yields a strict increase in the final strength under standard monotonicity assumptions used in the literature.
    
    Balance ensures that the semantics treats arguments that have the same structural influences (same attackers/supporters with the same base scores) symmetrically: if two arguments have identical incoming influences and identical incoming relations, then differences in their final strengths are entirely attributable to differences in their own base scores.
    
    Combining these two properties: since $a$ and $b$ receive the same aggregated influence from the rest of the framework (by hypothesis) and $\tau(a)>\tau(b)$ (by Axiom~\ref{axiom:1}), the monotonicity and balance of $\sigma$ imply
    \[
    \sigma(a) > \sigma(b).
    \]
    Thus, when $a\succ b$ and $a,b$ have equivalent supporters and attackers, the semantics $\sigma$ ranks $a$ strictly above $b$.
\end{proof}
\end{appendices}

\end{document}